\documentclass[11pt]{article}
\usepackage[utf8]{inputenc}
\usepackage{mystyle}

\begin{document}

\title {\huge Neural Tangent Kernel of Matrix Product States: Convergence and Applications}

\author
{Erdong Guo\thanks{University of California, Santa Cruz, email: \texttt{eguo1@ucsc.edu}} 
\qquad 
David Draper\thanks{University of California, Santa Cruz, email: \texttt{draper@ucsc.edu}} 
}

\date{}

\maketitle

\begin{abstract}
In this work, we study the Neural Tangent Kernel (NTK) of Matrix Product States (MPS) and the convergence of its NTK in the infinite bond dimensional limit. We prove that the NTK of MPS asymptotically converges to a constant matrix during the gradient descent (training) process (and also the initialization phase) as the bond dimensions of MPS go to infinity by the observation that the variation of the tensors in MPS asymptotically goes to zero during training in the infinite limit. By showing the positive-definiteness of the NTK of MPS, the convergence of MPS during the training in the function space (space of functions represented by MPS) is guaranteed without any extra assumptions of the data set. We then consider the settings of (supervised) Regression with Mean Square Error (RMSE) and (unsupervised) Born Machines (BM) and analyze their dynamics in the infinite bond dimensional limit. The ordinary differential equations (ODEs) which describe the dynamics of the responses of MPS in the RMSE and BM are derived and solved in the closed-form. For the Regression, we consider Mercer Kernels (Gaussian Kernels) and find that the evolution of the mean of the responses of MPS follows the largest eigenvalue of the NTK. Due to the orthogonality of the kernel functions in BM, the evolution of different modes (samples) decouples and the "characteristic time" of convergence in training is obtained.
\end{abstract}

\section{Introduction}\label{sec: intro}
Tensor networks (TN) are networks of finite or countable tensors connected by tensor contractions which originate from the study of Quantum Computation~\citep{feynman1985quantum, penrose1971applications}. 
Since quantum states are one-order tensors in Hilbert space and quantum operators (quantum gates) are high-order tensors, namely multi-linear maps on the products of Hilbert spaces (and their duals), it is natural to use TN (graph diagrams) to construct Quantum Circuits  which are widely used in Quantum Computation~\citep{ deutsch1989quantum}. 
Moreover, many-body quantum states can be efficiently approximated by TN with several special topological structures among which are Matrix Product States (MPS), Tensor Trains (TT), Tree Tensor Networks (TTN), etc.~\citep{biamonte2017tensor, jacot2018neural} 

The intuition why TN is an efficient description of quantum states is that the information on the correlation and the lattice geometry is more easily accessible in the "entanglement" representation. More interestingly, it is proposed that a new (holographic) dimension can be defined and geometry structures (curvature) emerge from the entanglement patterns\footnote{This is an implementation of the idea of holography which suggests that a quantum theory which encodes the information of the bulk where gravity theory is defined lives on the boundary of the bulk space time~\citep{hooft1993dimensional}.} in TN~\citep{van2009comments, swingle2012entanglement}.

It is found that the Tensor Decomposition is a computationally and statistically efficient tool to solve the inference problems~\citep{anandkumar2013tensor, anandkumar2014tensor}. And several TN based learning models suggested by the statistical learning community perform well in supervised (classification) and unsupervised (generative model) learning tasks~\citep{stoudenmire2016supervised, han2018unsupervised}. A lot of work has been done to extend the TN based learning models~\citep{novikov2015tensorizing, stoudenmire2018learning, huggins2019towards, reyes2021multi, guo2021bayesian} and also some interesting properties of TN as learning models (e.g. relation with other models, the representation power, the infinitely wide limit) have been explored~\citep{cohen2016expressive, huang2017neural, deng2017quantum, cai2018approximating, chen2018equivalence, clark2018unifying, glasser2020probabilistic, guo2021infinitely, guo2021representation, li2021boltzmann}, however the dynamics of TN during learning (training) process are not carefully analyzed. 

In this paper, we study the dynamics of MPS functions $\Psi(\mathbf{x})$ in the gradient descent process.
In Section~\ref{sec: mps_init_infinite} and~\ref{sec: essential}, we consider the infinite dimensional MPS and analyze its asymptotic behavior in the initialization and training phase. 
We derive the NTK of MPS which is the coefficient of the source term in the (stochastic) ODE which describes the evolution of MPS functions $\Psi(\mathbf{x})$ in Section~\ref{sec: ntk_mps_train}. 
By imaging each tensor $A^{s_{i}}_{\alpha_{i}\alpha_{i+1}}$ in MPS as an ensemble of neural layers $\{W_{(k)|\alpha_{i}\alpha_{i+1}}, k\in \{1, \cdots, |s_{i}|\}\}$, namely understanding the bond dimensions of MPS as the dimensions of the neural layers, then MPS are ensembles of fully-connected neural networks with all biases set to be zero and the activations set to be Identity functions. And the outputs of MPS are the weighted average of the outputs of all the neural networks in the ensembles produced by the MPS (Section~\ref{sec: mps_as_nn}). 
Based on the NTK formalism of MPS we developed, we study the Mean Square Error (MSE) Regression in Section~\ref{sec: mps_func_approx} and generative Born Machines in Section~\ref{sec: ntk_born}. 
We show that the NTK of MPS are positive definite which means the convergence of the gradient decent process is guaranteed. 
To show that the variation of the NTK $\Delta K$ during the training is asymptotic zero with respect to the infinite limit, we verify that the variation of the tensors $\{\Delta A^{s_{i}}_{\alpha_{i}\alpha_{i+1}}, i\in\{1, \cdots, n\}\}$ converges to zero in probability which is called "lazy training" phenomenon. For Born Machines, the partition function $Z[\Psi]$ which follows the chi-square distribution plays an interesting role in the dynamic equations. By taking the infinite length limit and also the infinite bond dimensional limit at the same time, the Stochastic ODE degenerates to the ODE due to the Weak Law of Large Number (WLLN). We get the analytical solution of the ODE of BM and analyze its properties in Section~\ref{sec: born_machine_solution}.  


\subsection{Dynamics of Neural Networks and NTK Theory}
The Neural Tangent Kernel (NTK) is a powerful tool to analyse the dynamics of the Artificial Neural Network (ANN) which achieved great success in a variety of statistical learning tasks [~\cite{jacot2018neural, lecun2015deep}]. In the continuous time setting, the networks functions evolve along the kernel gradient of the objective with respect to the NTK which is the Gram matrix of the Jocobian of the response with respect to the weights. 
By the same idea, the differential equations (ODE) describing the evolution of the weights and also the functional objective of the ANN can be written down based on which the interesting properties of their dynamics can be discovered. 

Because of the high non-linearity and the nested structure of the ANN which leads to the "coupling" of the weights in different layers, the neural ODE of the ANN is intractable. By taking the infinitely wide limit in each layer of the ANN sequentially, the NTK $K(\mathbf{x}^{(i)}, \mathbf{x}^{(j)})$ converges to a constant matrix $K_{\infty}(\mathbf{x}^{(i)}, \mathbf{x}^{(j)})$ in probability. With this simplification, the neural ODE can be solved analytically and the convergence of the optimization process is dominated by the principle component of the NTK. With the assumption of the boundedness of the integral of the training direction $d_{t}$, the variation of the weights during the training phase is asymptotic to zero which is the so-called "lazy training" phenomena of the infinitely wide ANN. The rate of the variation of the weights during training is of order $O(\frac{1}{\sqrt{n}})$ which induces the variation of the NTK of order $O(\frac{1}{\sqrt{n}})$ which means the NTK of infinitely wide ANN stays in constant during training phase. Moreover, it can be shown that the NTK is positively definite\footnote{The inputs of the kernel function are assumed to live on an unit sphere.} and then the convergence to the critical point by gradient descent is guaranteed in the wide limit which intuitively explains why the "over-parameterized" ANN still work well.

\subsection{Mathematical Preliminaries and Notations}
In this subsection, we introduce the NTK formalism and define the notations we will use in following sections.
For a learning model $\mathcal{M}$ with trainable parameters $\mathbf{\theta}$,
the function space $\mathcal{F}$ consists of all the functions represented by $\mathcal{M}$, namely $\mathcal{F} = \{f | f: \mathbb{R}^{n_{0}} \to \mathbb{R}^{n_{L}} \}$, where $n_{0}$ is the dimension of the training samples $\mathbf{x}$ and $n_{L}$ is the dimension of the outcomes of $\mathcal{M}$. We denote the realization functions of $\mathcal{M}$ as $F^{n_{L}}: \mathbb{R}^{p} \to \mathcal{F}$ which is a map from the parameter space $\mathbb{R}^{p}$ to the function space $\mathcal{F}$, where $p$ is the dimension of the parameter space. 
To analyse the dynamics of $\mathcal{M}$ in the optimization (gradient descent) process with respect to the cost function $\mathcal{L}: \mathcal{F} \to \mathbb{R}$, a bi-linear form $\langle\cdot, \cdot\rangle$: $\mathcal{F}\times\mathcal{F} \to \mathbb{R}$ on the function space $\mathcal{F}$ is introduced. To be precise, $\langle f, g \rangle_{K} = \mathbb{E}_{x, x^{\prime}\sim p^{\textit{in}}}[f(x)^{T}K(\mathbf{x}, \mathbf{x}^{\prime})g(x)]$, where $K(\mathbf{x}, \mathbf{x^{\prime}})$ is the symmetric NTK matrix. 

With the bilinear form, a map $\Phi_{K}: \mathcal{F}^{*} \to \mathcal{F}$ can be constructed, where $\mathcal{F}^{*}$ is the dual space of $\mathcal{F}$. Then the "kernel gradient" $\nabla_{K}\mathcal{L}$ can be obtained by mapping the functional derivative of the cost function $\partial_{f}{\mathcal{L}} = \langle d, \cdot \rangle_{p^{\textit{in}}} \in \mathcal{F}^{*}$ into the function space $\mathcal{F}$ using $\Phi_{K}(\cdot)$.  
Intuitively, we can understand $\nabla_{K}\mathcal{L}$ as the "velocity" of the networks functions evolve in the function space, and then we have
\begin{align}
    \frac{df(\mathbf{x})}{dt} = -\nabla_{K}\mathcal{L} = -\langle d_{f}, K(\mathbf{x}, \cdot) \rangle_{p^{\text{in}}}.
\end{align}
By the same idea, we can write down the ODE of the cost function $\mathcal{L}$ as 
\begin{align}
\label{eq: loss_dynamics}
    \frac{d\mathcal{L}}{dt} = -\langle d_{f}, \nabla_{K}\mathcal{L} \rangle_{p^{\text{in}}}.
\end{align}
The NTK can be obtained as  
\begin{align}
    K_{(lm)}(\mathbf{x}^{(i)}, \mathbf{x}^{(j)}) = \sum_{p}\frac{\partial{f^{(l)}(\mathbf{x}^{(i)})}}{\partial{W_{p}}}\otimes\frac{\partial{f^{(m)}(\mathbf{x}^{(j)})}}{\partial{W}_{p}},
\end{align}
where the ANN functions $f^{(l)}(\mathbf{x})$ are parameterized as 
\begin{align}
f^{(l)}(\mathbf{x}; \mathbf{W}) = \frac{1}{\sqrt{n^{L}}}W^{[L]}\cdot\sigma(\frac{1}{\sqrt{n^{L-1}}}W^{[L-1]}\cdot\sigma(\cdots\sigma(\frac{1}{\sqrt{n^{1}}}W^{[1]}\mathbf{x} + \beta^{[1]}))) + \beta^{L}.
\end{align}
We note here that for the $(i, j)$ component of neural tangent kernel $K_{ij}$ is defined on the sample points pair $(\mathbf{x}_{i}, \mathbf{x}_{j})$ and it is a $n_{L}\times n_{L}$ matrix. The rescaling factor $\frac{1}{\sqrt{n_{i}}}$ in $i$'th layer is crucial to get an consistent asymptotic behavior. 
Since all the terms in above sum are independent and identically distributed (i.i.d.), the NTK $K_{ij}(\cdot, \cdot)$ converge to the constant matrix as the widths of ANN go to infinity by the weak law of large numbers (w.l.l.n.).


\section{MPS with Infinite bond dimensions}
\label{sec: mps_init_infinite}

\subsection{Infinitely dimensional MPS as Gaussian Process}
We consider the set up of MPS as follows, 
\begin{align}
\label{eq: mps_outputs}
    \Psi(\mathbf{x}^{(i)}; \mathbf{A}) = \sum_{\{s, \alpha\}}A^{s_{1}}_{\alpha_{1}\alpha_{2}}\cdots A^{s_{i}}_{\alpha_{i}\alpha_{i+1}}\cdots A^{s_{n}}_{\alpha_{n}\alpha_{1}}\Phi^{s_{1}\cdots s_{n}}(\mathbf{x}^{(i)}),
\end{align}
where $\Phi^{s_{1}\cdots s_{n}}(\mathbf{x})=\otimes_{i}^{n}\phi^{s_{i}}(x_{i})$ is the kernel function.

In~\cite{guo2021infinitely}, it is proposed that as the dimensions of the MPS go to infinity, the functions $\Psi^{l}(\mathbf{x}; \mathbf{A})$ represented by MPS converge to the Gaussian Process (GP). 
Since the rich structure of the indices in MPS, the asymptotic GP can be realized by several schemes of limit processes. Here a GP limit process is proposed to prepare for the NTK analysis in next section.
We give our first theorem on the GP induced by the infinitely dimensional MPS as follows,

\begin{theorem}
\label{the: infinite_mps_gp}
For a data set $\{(\mathbf{x}^{(i)}, \mathbf{y}^{(i)}), i\in\{1, \cdots, m\}\}$, the outcomes $\Psi(\mathbf{x})$ of MPS defined in \ref{eq: mps_outputs} converge to Normal random variables as the bond dimensions $\alpha_{1}, \cdots, \alpha_{n} \to \infty$ sequentially. 
Then MPS functions $\Psi(\cdot)$ converge to the Gaussian Process, namely 
\begin{align}
    &\Psi \sim \text{GP}(\mu, \Sigma),\\
    &\mu = \mathbf{0},\\
    &\Sigma(\mathbf{x}, \mathbf{x}^{\prime}) = \prod_{i}|s_{i}|\sigma_{i}^{2}\phi^{i}(x_{1})\cdot\phi^{i}(x^{\prime}_{1})
\end{align}
where $\mu(\cdot)$ is the mean function and $\Sigma(\cdot, \cdot)$ is the covariance function, as the bond dimensions go to infinity. 
\end{theorem}
\begin{remark}
We note here the variances of the distributions followed by the tensors $A^{s_{i}}_{\alpha_{i}\alpha_{i+1}}$ are rescaled by the factors $\frac{1}{\sqrt{\alpha_{i}\alpha_{i+1}}}$ as we take the infinite limit of the bond dimensions of the MPS sequentially. The intuition for the rescaling factor is to asymptotically decrease the "contribution" to the outcome of each tensor in the tensor chain increases to achieve an non-trivial limit as the number of tensors goes to infinity. We show the proof in Appendix~\ref{pf: proof_of_the_1}. 
\end{remark}
\subsection{Relations with the ANN}
\label{sec: mps_as_nn}
It is already shown that MPS is equivalent to Neural Networks equipped with kernel functions in~\cite{guo2021representation} by contracting all the bond dimensions between each adjacent tensors in MPS. Here we reserve the bond dimensions in MPS and view the tensor $A^{s_{i}}_{\alpha_{i}\alpha_{i+1}}$ as a $s_{i}$ dimensional weights $A_{\alpha_{i}\alpha_{i+1}}$. From this perspective, MPS is equivalent to a weighted average of an ensemble of fully-connected ANN with identity activation functions.
\begin{proposition}
For a MPS with the bond dimension $\alpha_{1} = 1$, then 
\begin{align}
    &\Psi(\mathbf{x}; \mathbf{A}) = \sum_{i}W_{i}(\mathbf{x})N_{i}(\mathbf{A}),\\ 
    &N_{i}(\mathbf{A}) = W^{[i_{1}]}_{1\alpha_{2}}\sigma(W^{[i_{2}]}_{\alpha_{2}\alpha_{3}}\cdots \sigma(W^{[i_{n}]}_{\alpha_{n}1})),\\
    &W_{i}(\mathbf{x}) = \Phi^{i_{1}\cdots i_{n}}(\mathbf{x}),
\end{align}
where $W^{[i_{k}]}_{\alpha_{k}\alpha_{k+1}}$ is the $i_k$ component of $A^{s}_{\alpha_{k}\alpha_{k+1}}$ in the bond dimension, and $\sigma(\cdot)$ is the identity activation.
\end{proposition}
\begin{remark}
The bias of all the neural networks $N_{i}$ are set to be zero. 
\end{remark}
For a MPS with $n$ tensors, the cardinality of the ANN ensemble $\mathcal{N} = \{N_{i}(\mathbf{A}), i\in\otimes_{j}^{n}s^{i}\}$ is  $|s|^{n}$ which is the same as the dimension of $W_{i}(\mathbf{x})$. A pair of neural networks $N_{i}$ and $N_{j}$ may correlate with each other if common tensors are shared. But in the infinite bond dimensional limit, all the neural networks become independent with each. 

\begin{figure}[H]
  \centering
  \includegraphics[width=.9\linewidth]{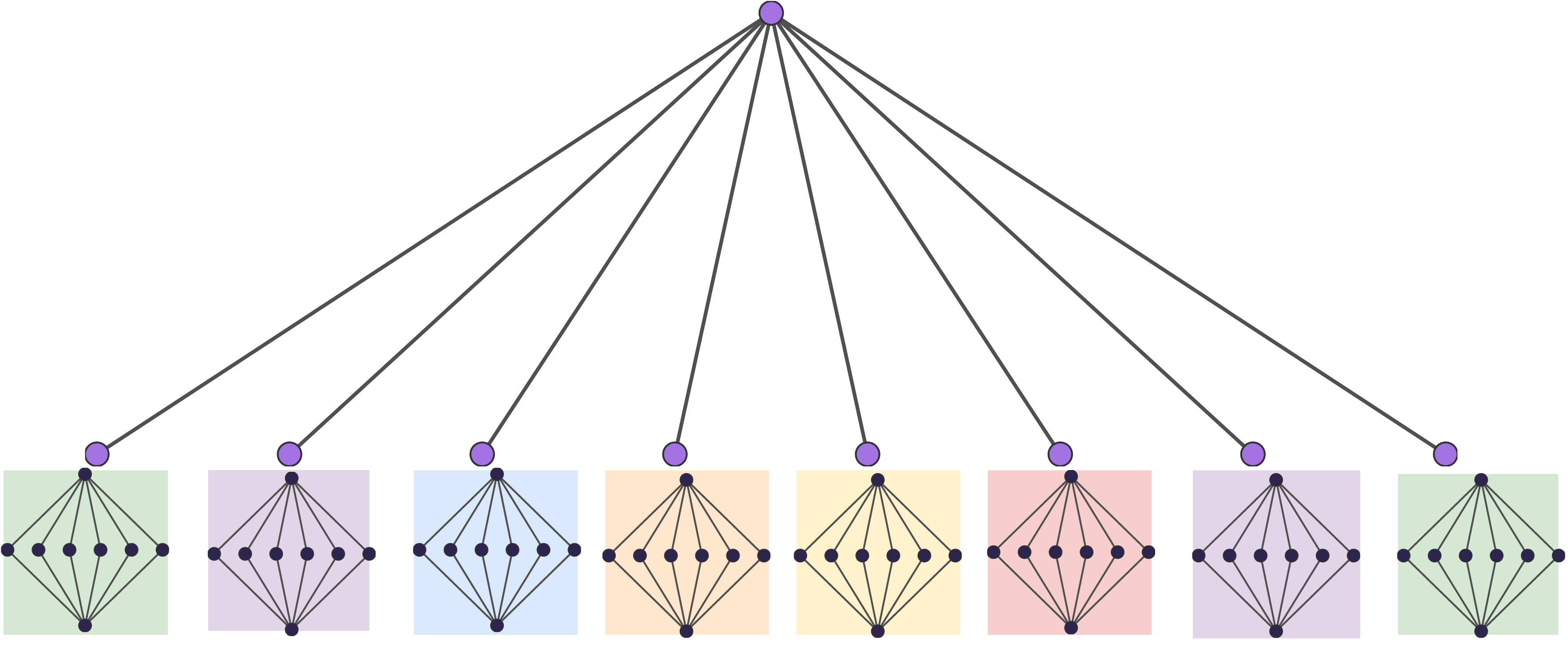}  
  \caption{We show a MPS with three tensors and in each tensor the dimension of the index $s_{i}$ is two which leads to eight neural networks in the ensemble $\mathcal{N}$.  The last layer of neurons $W_{i}$ is induced by the kernel function $\Phi^{s_{1}s_{2}s_{3}}(\mathbf{x})$ and the outputs of the MPS is obtained by averaging all the outcomes by $N_{i} \in \mathcal{N}$ according to the weights in $W_{i}$.}
  \label{fig: ntk_mps_ensemble}
\end{figure}



%


\section{NTK of MPS and its Limit in Infinite Bond Dimension}\label{sec: essential}

\subsection{Dynamics of MPS}\label{subsec: ntk_theory}
\label{sec: mps_dynamics_infinite}
In this section, we consider the NTK of MPS defined as 
\begin{align}
\label{eq: ntk_mps}
    K(\mathbf{x}^{(i)}, \mathbf{x}^{(j)}) = \sum_{\{s, \alpha\}}\eta_{\alpha_{i}\alpha_{i+1}}\odot(\frac{\partial{\Psi(\mathbf{x}^{(i)})}}{\partial{\mathbf{A}^{s_{i}}_{\alpha_{i}\alpha_{i+1}}}}\otimes\frac{\partial{\Psi(\mathbf{x}^{(j)})}}{\partial{\mathbf{A}^{s_{i}}_{\alpha_{i}\alpha_{i+1}}}}).
\end{align}
Different from the definition of NTK in the fully-connected ANN suggested in~\citep{jacot2018neural}, learning rate $\eta_{\alpha_{i}\alpha_{j}}$ is introduced to rescale the Gram Matrix of MPS to achieve an appropriate limit as the bond dimensions go to infinity.  

Actually NTK does not depend on the objective of the "learning" model which means it is task independent. For the convergence analysis of the NTK in the initialization phase, the objective will not play a role in the theory. However, objective will be critical in controlling the convergence behavior in the training period. To keep the NTK to be asymptotically constant in the training process, the order of the variation of the NTK and also the tensors in MPS should be asymptotically zero where the boundedness of the gradient direction $d_{\Psi}(\cdot)$ is important.  

By the NTK introduced above, the dynamics of the responds of MPS and the objective can be written down as 
\begin{align}
\label{eq: tensor_dynamics}
    &\frac{d}{dt}\Psi = -\nabla_{K}\mathcal{L},\\ 
    &\frac{d}{dt}\mathcal{L} = -\langle d, \nabla_{K}\mathcal{L} \rangle_{p^{\text{in}}},
\end{align}
where $\mathcal{L}(\mathbf{x}, \mathbf{A})$ is the objective for a specific task and $\langle d, \cdot\rangle = \partial_{\Psi}\mathcal{L}$. 


\subsection{NTK of MPS: Initialization}\label{ntk_mps_init}
Our theorem on the asymptotic behavior of the NTK of MPS as the bond dimensions go to infinity in the initialization period is as follows,

\begin{theorem}
For a MPS with following set-up as $A^{s_{i}}_{\alpha_{i}\alpha_{i+1}} \sim$ $\mathcal{N}(\mathbf{0}, \frac{\sigma_{i}^{2}}{\sqrt{|\alpha_{i}||\alpha_{i+1}|}}\mathbb{I}^{s_{i}}_{\alpha_{i}\alpha_{i+1}}))$, as the bond dimensions $\{\alpha_{i}, i\in\{1, \cdots, n\}\}$ goes to infinity consequentially, the NTK $K_{ij}(t)$ of MPS converges to a static matrix in probability, 
\begin{align}
    K_{ij}(t) \xrightarrow[]{\text{Prob.}}\sum_{k}\phi(x^{(i)}_{k})\cdot\phi(x^{(j)}_{k})\prod_{l=1; l \neq k}^{n}\sigma_{l}^{2}\phi(x^{(i)}_{l})\cdot\phi(x^{(j)}_{l}),
\end{align}
where $\mathbb{I}^{s_{i}}_{\alpha_{i}\alpha_{j}} = \mathbb{I}^{s_{i}}\otimes\mathbb{I}_{\alpha_{i}}\otimes\mathbb{I}_{\alpha_{i+1}}$.
\end{theorem}
\begin{remark}
To control the infinite limit process and then achieve a non-trivial limit, we need to tune the decreasing rate of learning rate appropriately. The learning rate $\eta$ should be defined on each element of NTK individually as
\begin{align}
    \eta_{\alpha_{i}\alpha_{i+1}} = (|\alpha_{i}||\alpha_{i+1}|)^{-1/2}.
\end{align}
The proof of above theorem is in the Appendix~\ref{pf: proof_of_the_2}
\end{remark}
By optimizing the objective $\mathcal{L}$ along the kernel (NTK) gradient direction, the critical point of MPS can be discovered since NTK of MPS is positively definite as proposed in the following proposition,
\begin{proposition}
\label{prop: ntk_positive}
The NTK $K_{ij}$ of the infinite bond dimensional MPS is positively definite.
\end{proposition}
By Mercer's condition, we can show that the NTK is positively-definite as proposition~\ref{prop: ntk_positive}. The proof is in Appendix~\ref{pf: ntk_mps_positive}. 

\subsection{NTK of MPS: Training}
\label{sec: ntk_mps_train}
In this part, we will study the evolution of NTK of MPS in the training process. Similar to the result in the infinitely wide ANN, with mild assumptions, we can show the NTK is also asymptotically static in the training phase.
\begin{theorem}
Assume that the integral of the training direction $\int_{0}^{T}{d_{t}(\cdot)dt}$ is bounded in arbitrary time period $[0, T]$, as the bond dimensions of MPS go to infinity, the NTK of MPS then converges to a constant matrix as 
\begin{align}
\label{eq: ntk_mps_training}
    K_{i, j}(\mathbf{x}_{i}, \mathbf{x}_{j}) \xrightarrow[]{\text{Prob.}}\sum_{k}\phi(x^{(i)}_{k})\cdot\phi(x^{(j)}_{k})\prod_{l=1; l \neq k}^{n}\sigma_{l}^{2}\phi(x^{(i)}_{l})\cdot\phi(x^{(j)}_{l}).
\end{align}
And the dynamics of the outputs of MPS $\Psi(\mathbf{x})$ follow the differential equation as 
\begin{align}
    \frac{d}{dt}\Psi(\cdot; \mathbf{A})) = \Phi_{K}(\langle d_{t}(\cdot; \mathbf{A}), \rangle).
\end{align}
\end{theorem}

To show that NTK of MPS is asymptotically constant, we need following lemma~\ref{le: lazy_training} which describes the "lazy" training phenomena in infinite MPS. 
\begin{lemma}
\label{le: lazy_training}
For MPS with settings as above, as the bond dimensions go to infinity, we have following relation 
\begin{align}
    \lim_{\alpha_{1}, \cdots, \alpha_{n}\to\infty}\sup_{t}{|A^{s_{i}}_{\alpha_{i}\alpha_{i+1}}(t) - A^{s_{i}}_{\alpha_{i}\alpha_{i+1}}(0)|} \xrightarrow[]{\text{Prob.}} 0.
\end{align}
\end{lemma}
Above lemma says that the tensors in MPS "freeze up" in the infinite limit since the variation of the tensors converges to zero. Although the update of the tensors is asymptotically to zero, the outputs of MPS still can "learn" due to the collective contribution of the updates of infinite tensors is not a infinitely small number. We show the proof in the Appendix~\ref{pf: pf_of_zero_variations}.

\subsection{Functions Approximation by MPS}
\label{sec: mps_func_approx}
For RMSE, we can write down the objective as 
\begin{align}
    \mathcal{L}(\mathbf{x}; \mathbf{A}) = \sum_{i=1}^{m}(\Psi(\mathbf{x}^{(i)}; \mathbf{A}) - y^{(i)})^{2}.
\end{align}
And obviously the dynamics of the tensors $A^{s_{i}}_{\alpha_{i}\alpha_{j}}$ are as
\begin{align}
    \label{eq: dynamics_mps_func_approx}
    \frac{d}{dt}\Psi(\mathbf{x^{(i)}; \mathbf{A}}) &= -\sum_{j}K_{i, j}(\mathbf{x}^{(i)}, \mathbf{x}^{(j)}) d_{j}(\mathbf{x}^{(j)}; \mathbf{A}),
\end{align}
where
\begin{align}
    \label{eq: ntk_mps_func_approx}
    &K_{i, j}(\mathbf{x}^{(i)}, \mathbf{x}^{(j)}) = n\prod_{l=1}^{n}\phi(x^{(i)}_{l})\cdot \phi(x^{(j)}_{l}),\\
    &d_{j}(\mathbf{x}^{(j)}; \mathbf{A}) = \Psi(\mathbf{x}^{(j)}; \mathbf{A}) - y^{(j)}.
\end{align}
Without loss of generality, the tensors $A^{s_{i}}_{\alpha_{i}\alpha_{i+1}}$ in above result are initialized with zero mean and unit variance iid distributions. It is easy to show that the training direction $d_{j}(\cdot)$ is bounded in probability and then the stationary of $K_{ij}(\cdot, \cdot)$ and "lazy" training follow. 

Since NTK of MPS~\ref{eq: ntk_mps_func_approx} is static, we can get the solution of~\ref{eq: dynamics_mps_func_approx} as
\begin{align}
    \vec{\Psi}(t) = \mathbf{y} + (\vec{\Psi}(0) - \mathbf{y})\exp{(-tK)}, 
\end{align}
where $\vec{\Psi}(t)$ is the vector of the outputs of MPS and $\mathbf{y}$ is the vector of labels on the training data set.
Actually the NTK of MPS~\ref{eq: ntk_mps_training} can be represented by the product of a series of positive definite Gram matrices\footnote{Here the Gram matrix is a $1\times1$ matrix, namely a scalar, since we only consider the MPS with one dimensional output. It is straightforward to extend our work to multi-dimensional case.}, namely the Mercer kernels. Since the kernel function $\Phi^{s_{1}\cdots s_{n}}(\cdot)$ is factorized as the product of series of kernel function on each feature space $\phi^{s_{1}}(\cdot)\otimes\cdots\phi^{s_{i}}(\cdot)\cdots\otimes\phi^{s_{n}}(\cdot)$, the Mercer's kernel $k^{(i)}(\cdot, \cdot)$ is defined on each feature space $\{\mathbf{x}^{(j)}_{i}, j\in\{1, \cdots, m\}\}$ individually as
\begin{align}
\phi(\mathbf{x}^{(j)}_{i})\cdot\phi(\mathbf{x}^{(l)}_{i}) = k^{(i)}(\mathbf{x}^{(j)}_{i}, \mathbf{x}^{(l)}_{i}).     
\end{align}
In the following example, we will consider the Gaussian Kernel and analyze the properties of the corresponding solution.
\begin{example}
We assume the Mercer's Kernel in each feature space to be the Gaussian Kernel.
More specifically, for the $i$th feature space $\{x^{(j)}_{i}, j\in\{1, \cdots, m\}\}$, we define $\phi(x^{(j)}_{i})\cdot\phi(x^{(l)}_{i}) = \exp{(-\frac{1}{2}\frac{(x^{(j)}_{i} - x^{(l)}_{i})^{2}}{\tau_{i}^{2}})}$, 
then we have 
\begin{align}
    \prod_{i=1}^{n}\phi(x^{(j)}_{i})\cdot\phi(x^{(l)}_{i}) = \exp{(-\frac{1}{2}(\mathbf{x}^{(j)}-\mathbf{x}^{(l)})\Sigma^{-1}(\mathbf{x}^{(j)}-\mathbf{x}^{(l)}))},
\end{align}
where $\Sigma = \tau_{1}\oplus\cdots\tau_{i}\cdots\oplus\tau_{n}$.
\end{example}
Here we assume that the distance between arbitrary two sample points in the training data set is the same, then we know that all the diagonal elements of the NTK matrix $K_{ij}$ are one and all the off-diagonal elements are $r$ ($0 < r < 1$). So the mean of all the responds of MPS on training data set $\bar{\Psi} = \frac{1}{m}\sum_{i}\Psi(\mathbf{x}^{(i)})$ evolves along component of the biggest eigenvalue of the NTK, 
\begin{align}
    &\bar{\Psi}(t) = \bar{\mathbf{y}} + (\bar{\Psi}(0) - \bar{\mathbf{y}})\exp{(-t(1 + (m-1)r))},
\end{align}
where $\bar{\mathbf{y}}$ is the mean of the labels in the data set.

\section{NTK of Born Machines}
\label{sec: ntk_born}

\subsection{Introduction to Born Machines}\label{subsec: born_introduction}
Born Machines (BM) are a type of generative models inspired by the wave functions in the Quantum Mechanics~\cite{han2018unsupervised}. Different from Boltzman Machines~\cite{ackley1985learning}, there are no latents in BM and "probability amplitude" for a given sample point $\mathbf{x}$ is estimated by the MPS by the product of a chain of one-particle state.

For the Born Machine, the objective $\mathcal{L}(\{\mathbf{x}\})$ is the negative log-likelihood (NLL) function as
\begin{align}
\label{eq: nll_born_machine}
    &\mathcal{L}(\{\mathbf{x}\}) = 
    -\sum_{i}\log{|\Psi(\mathbf{x}^{(i)})|^{2}} + m\log{Z},\\
    \label{eq: partition_born_machine}
    &Z[\Psi] = \sum_{\mathbf{x}\in\Omega}|\Psi(\mathbf{x})|^{2},\\
    &\Omega = \{0, 1\}^{\otimes n},
\end{align}
where $n$ is the length of the tensor chains of MPS and all the sample points in the data set $\Omega$ are vectors with components to be one or zero. We use the same $\Psi(\mathbf{x})$ in~\ref{eq: mps_outputs} and set the kernel function $\Phi^{s_{1}, \cdots, s_{n}}(\mathbf{x})$ to be 
\begin{align}
\label{eq: kernel_born_machine}
\Phi^{s_{1}, \cdots, s_{n}}(\mathbf{x}) = \otimes_{1}^{n}\frac{1}{\sqrt{2}}[x_{i}, 1 - x_{i}].
\end{align}
It is crucial to use the squared outcomes $|\Psi(\cdot)|^{2}$ to represent the likelihood of the sample points instead of the outcome directly according to Max Born's statistical interpretation of wave functions.   
\subsection{Dynamics of Born Machines in Training}
\label{sec: born_machine_solution}
As we mentioned before, the NTK only depends on the network structure instead of the objective. Here we write down the NTK of BM in the following Proposition~\ref{prop: ntk_born_machine},
\begin{proposition}
\label{prop: ntk_born_machine}
Considering above settings of BM, the NTK is as follows,
\begin{align}
    K(\mathbf{x}^{(i)}, \mathbf{x}^{(j)}) = \delta_{ij}\prod_{k}^{n}\sigma_{k}^{2},
\end{align}
where $n$ is the length of the tensor chain, and $\sigma_{k}^{2}$ is the variance of the $A^{s_{i}}_{\alpha_{i}\alpha_{i+1}}$.
\end{proposition}

However, the training direction $\langle d_{j}(\mathbf{x}^{(j)}, \cdot\rangle$ in~\ref{eq: d_born_machine} is determined by the objective~\ref{eq: nll_born_machine} and we write down its expression:
\begin{align*}
\label{eq: d_born_machine}
    d_{j}(\mathbf{x}^{(j)}) = -\frac{\delta\mathcal{L}(\{\mathbf{x}\})}{\delta\Psi(\mathbf{x}^{(j)})}
    = 2(\frac{1}{\Psi(\mathbf{x}^{(j)})} - m\frac{\Psi(\mathbf{x}^{(j)})}{Z}),
\end{align*}
So we propose our first proposition on the boundedness of $d_{j}(\cdot)$:
\begin{proposition}
\label{prop: boundedness_born_machine}
For BM with the objective as~\ref{eq: nll_born_machine}, the training direction functional $\langle d(\cdot)|_{\Psi(\mathbf{x}^{(j)})}, \cdot\rangle$ is bounded in probability with the assumption that $Z[\Psi]$ is bounded in probability. 
\end{proposition}

Interestingly, we can find that as the bond dimensions go to infinity, the correlation of the responds of BM with different sample points decay asymptotically to zero due to the orthogonality of the kernel function used in~\ref{eq: kernel_born_machine}. This means that the GP induced by infinite (bond dimensional) MPS will "degenerate" to a series of independent Normal random variables as the case in~\ref{prop: ntk_born_machine} where all the off-diagonal matrix elements are zero.

The partition function $Z[\Psi]$ gets into the training direction $d_{j}(\cdot)$ as~\ref{eq: d_born_machine} and it couples the evolutions of the responds of MPS $\Psi(\mathbf{x})$ of different sample points together which will lead to complicated non-linear behaviors of the (Stochastic) ODE system. However, as we know the outcomes $\Psi(\mathbf{x}^{(i)})$ of BM become asymptotically independent in the infinite bond dimensional limit, we can get the analytical form of the partition function $Z[\Psi]$ which follows the Gamma distribution as proposed in following Proposition~\ref{prop: partition_distribution},
\begin{proposition}
\label{prop: partition_distribution}
For BM set up as~\ref{eq: partition_born_machine}, the partition function $Z[\Psi]$ follows the Gamma distribution:
\begin{align}
    Z[\Psi] \sim \Gamma(2^{n-1}, 2\prod_{i}^{n}\sigma_{i}^{2}).
\end{align}
Specially, if the length of the tensor chains go to infinity, then $\frac{Z[\Psi]}{2^{n}}$ converges to a constant, namely  
\begin{align}
  \frac{Z[\Psi]}{2^{n}} \xrightarrow[]{\text{Prob.}} \prod_{i}^{n}\sigma_{i}^{2}.
\end{align}
The proof is in Appendix~\ref{pf: partition_function_bm}.
\end{proposition}
It is easy to write down the dynamics of BM with respect to the training direction in~\ref{eq: d_born_machine} as
\begin{align}
\label{eq: dynamics_bm}
    \frac{d}{dt}\Psi(\mathbf{x}^{(i)}) = \sum_{j}K(\mathbf{x}^{(i)}, \mathbf{x}^{(j)})d(\mathbf{x}^{(j)}) = 2\sum_{j}K(\mathbf{x}^{(i)}, \mathbf{x}^{(j)})(\frac{1}{\Psi(\mathbf{x}^{(i)})} - m\frac{\Psi(\mathbf{x}^{(i)})}{Z[\Psi]}).
\end{align}
With the NTK as~\ref{eq: ntk_mps_training} and the partition function $Z[\Psi]$~\ref{eq: partition_born_machine}, we can solve the (Stochastic) ODE~\ref{eq: dynamics_bm} analytically as 
\begin{align}
\label{eq: born_machine_one_sample_solution}
    &\Psi(t) = \pm\sqrt{(\Psi_{0}^{2} - \frac{Z}{m})\exp{(-\frac{4mK}{Z}t)} + \frac{Z}{m}},\\
    \label{eq: pt_evolution}
    &P_{\mathbf{x}}(t) = \frac{1}{m} - (\frac{1}{m} - P_{\mathbf{x}}(0))\exp{(-\frac{4mK}{Z}t)}.
\end{align}
It is obvious that $K(\cdot, \cdot)$ and $Z[\Psi]$ are both positive, so as $t \to \infty$, $P_{\mathbf{x}}(t)$ converges to $\frac{1}{m}$. This means that during the "learning" process, BM "memorize" the training samples by increasing the probability of the samples BM views and eventually an "uniform" distribution is learned with equal probability on each sample. According to~\ref{eq: pt_evolution}, the "characteristic time" $T_{\text{Learning}}$ is $\frac{Z}{4mK}$ which represents the order of the "training time". Since $K(\cdot, \cdot)$ is diagonal and also all the diagonal elements are the same in the infinite limit, it means that all the responds of MPS $\Psi(\mathbf{x})$ evolve individually and also with the same dynamics. Actually we can estimate the training time by constructing confidence interval of "characteristic time", however using the mean of the partition function $Z$ we calculate the training time as
\begin{align}
    T_{\text{Learning}} = \frac{2^{n-2}}{m}.
\end{align}
Moreover, if we consider the length limit and assume that the training size $m$ is of order $O(2^{n})$, then we get $T_{\text{Learning}} = \frac{1}{4}$ which means that the BM converges in constant time although the training size is of order $O(2^{n})$ with big $n$. 

Unlike the learning in BM analyzed here, different principle components of the network function of infinitely wide fully-connected neural networks evolve in different rate because of the non-zero off-diagonal elements in NTK induced by the correlation of the sample points. By this observation, the early-stopping is suggested to avoid over-fitting. From these analysis, we can conclude the advantage of BM is that each sample points evolve individually with the same ratio which means all the "modes" in training set are well preserved and also over-fitting problem is naturally avoided, but the disadvantage is that noise sample might affect the learning process which cannot happen in ANN since noise samples have small eigenvalues which lead to slow convergence.     

\section{Conclusion}
We study the dynamics of MPS and its infinite limit by the NTK formalism. For MPS initialized by IID Normal distributions, MPS functions $\Psi(\cdot)$ converge to the GP as the bond dimensions of MPS $\{\alpha_{i}, i\in \{1, \cdots n\}\}$ go to infinity. To connect the infinite bond dimensional limit of MPS with the convergence of infinite wide ANN, we show that MPS is equivalent to the weighted average of an ensemble of fully-connected linear neural networks. In the training process, it is shown that the NTK of MPS keeps asymptotically to be fixed by which we can solve the ODE analytically. Interestingly, we find that Mercer's kernels induced by the kernel of the training data points get involved in the NTK of MPS. For the functions approximation case, we consider the Gaussian kernel and show that the mean of the outcomes of MPS follows the greatest principle of the NTK. For the unsupervised task, we analyze BM and obtain the evolution equation of the probability of each sample mode which is not correlated with each other. It is found that the increase of number of tensors in MPS leads to the exponential increase in learning time, but the increase of sample points decreases the learning time with ratio $O(\frac{1}{m})$. As is shown above, the Mercer Kernel in NTK is defined as a product of a series of Mercer kernels on each feature space. This can be extended by introducing the coupling of the features into the kernel functions.

\section*{Acknowledgements}
We would like to thank all the people who provided us with their helpful discussions and comments.

\bibliographystyle{./ims}

\bibliography{reference.bib}

\clearpage
\section*{A. GP by Infinite Bond Dimensional MPS}
In this part, we prove that as the bond dimensions of MPS go to infinity, the outputs $\Psi(\mathbf{x})$ converge to Normal random variables which means $\Psi(\cdot)$ belongs to the $GP$ as the statement in Theorem~\ref{the: infinite_mps_gp}. Before the proof of~\ref{the: infinite_mps_gp}, we need one lemma and several propositions. We prove \textbf{Lemma $1$} firstly.

\noindent
{\bf Lemma 1.}
{\it
For a tensor chain with two tensor nodes $\{A^{s_{1}}_{\alpha_{1}\alpha_{2}}, A^{s_{2}}_{\alpha_{2}\alpha_{3}}\}$ which follow iid Normal distributions as $A^{s_{1}}_{\alpha_{1}\alpha_{2}} \sim \mathcal{N}(\mathbf{0}, \frac{\sigma_{1}^{2}}{\sqrt{|\alpha_{1}||\alpha_{2}|}}\mathbb{I}^{s_{1}}_{\alpha_{1}\alpha_{2}})$ and $A^{s_{2}}_{\alpha_{2}\alpha_{1}} \sim \mathcal{N}(\mathbf{0}, \frac{\sigma_{2}^{2}}{\sqrt{|\alpha_{2}||\alpha_{1}|}}\mathbb{I}^{s_{2}}_{\alpha_{2}\alpha_{1}})$ where $\sigma_{1}$, $\sigma_{2}$, $|S_{1}|$ and $|S_{2}|$ are all finite, we introduce a tensor $B^{s_{1}s_{2}} = \sum_{\{\alpha_{1}, \alpha_{2}\}}A^{s_{1}}_{\alpha_{1}\alpha_{2}}A^{s_{2}}_{\alpha_{2}\alpha_{1}}$. Then $B^{s_{1}s_{2}}$ follows the multi-variate Normal distribution as 
\begin{align}
    B^{s_{1}s_{2}} &\sim \mathcal{N}(\mathbf{0}, \sigma_{1}^{2}\sigma_{2}^{2}\mathbb{I}^{s_{1}s_{2}}), 
\end{align}
as the bond dimensions $|\alpha_{2}|$ and $|\alpha_{1}|$ go to infinity sequentially. 
}
\begin{proof}
Since tensor $B^{s_{1}s_{2}}$ is the contraction of two i.i.d. Gaussian random variables $A^{s_{1}}_{\alpha_{1}\alpha_{2}}$ and $A^{s_{2}}_{\alpha_{2}\alpha_{1}}$, we can get the mean and the variance of the contraction $B^{s_{1}s_{2}}$ as 
\begin{align}
    &\mathbb{E}[B^{s_{1}s_{2}}] = 0, \\
    &\mathbb{V}[B^{s_{1}s_{2}}] = |\alpha_{1}||\alpha_{2}|\mathbb{V}[A^{s_{1}}_{\alpha_{1}\alpha_{2}}A^{s_{2}}_{\alpha_{2}\alpha_{1}}]
    =\sigma_{1}^{2}\sigma_{2}^{2}.
\end{align}
We know $B^{s_{1}s_{2}}$ is the sum of iid components, namely $A^{s_{1}}_{ij}A^{s_{2}}_{jk} \indep A^{s_{1}}_{il}A^{s_{2}}_{lk}$, then by central limit theorem, if $\alpha_{2}, \alpha_{1}\to\infty$, 
\begin{align}
    B^{s_{1}s_{2}} \xrightarrow[]{\text{Dist.}} \mathcal{N}(\mathbf{0}, \sigma_{1}^{2}\sigma_{2}^{2}\mathbb{I}^{s_{1}s_{2}}).
\end{align}
We note that different components of $B^{s_{1}s_{2}}$ are i.i.d.  
\end{proof}
By mathematical induction, we can extend above lemma to tensor chains with arbitrary lengths, namely $B^{s_{1}\cdots s_{n}}$. So we get following proposition, 

\vspace{0.1in}
\noindent
{\bf Proposition 1.}
\label{prop: tensor_chain_convergence}
{\it 
For a tensor chain $B^{s_{1}\cdots s_{n}}$ with $n$ tensors initialized by iid Normal distributions, namely $A^{s_{i}}_{\alpha_{i}\alpha_{i+1}} \sim \mathcal{N}(\mathbf{0}, \frac{\sigma_{i}^{2}}{\sqrt{\alpha_{i}\alpha_{i+1}}}\mathbb{I}^{s_{i}}_{\alpha_{i}\alpha_{i+1}})$, as $\alpha_{1}, \cdots, \alpha_{n} \to \infty$, 
\begin{align}
B^{s_{1}\cdots s_{n}} \xrightarrow[]{\text{Dist.}} \mathcal{N}(\mathbf{0}, (\prod_{i}^{n}\sigma_{i}^{2})\mathbb{I}^{s_{1}\cdots s_{n}})),  
\end{align}
where $\{\sigma_{i}, i\in\{1, \cdots n\}\}$ and $\{|s_{i}|, i\in\{1, \cdots n\}\}$ are all finite.
}
\begin{proof}
To prove above theorem, we just need to show that as $\alpha_{2}, \cdots, \alpha_{n} \to \infty$,
\begin{align}
 B^{s_{1}\cdots s_{n}}_{\alpha_{1}\alpha_{1}} \xrightarrow[]{\text{Dist.}} \mathcal{N}(\mathbf{0}, \frac{1}{\alpha_{1}}(\prod_{i}^{n}\sigma_{i}^{2})\mathbb{I}^{s_{1}\cdots s_{n}}_{\alpha_{1}\alpha_{1}}). 
\end{align}
More generally, we can show that
\begin{align}
\label{eq: mps_distribution}
  B^{s_{1}\cdots s_{n}}_{\alpha_{1}\alpha_{2}} \xrightarrow[]{\text{Dist.}} \mathcal{N}(\mathbf{0}, (\alpha_{1}\alpha_{2})^{-\frac{1}{2}}(\prod_{i}^{n}\sigma_{i}^{2})\mathbb{I}^{s_{1}\cdots s_{n}}_{\alpha_{1}\alpha_{2}}).       
\end{align}
We prove \ref{eq: mps_distribution} by induction on the number $n$ of tensors in the chain. 
When $n=1$, $B^{s_{1}}_{\alpha_{1}\alpha_{2}} = A^{s_{1}}_{\alpha_{1}\alpha_{1}}$, then \ref{eq: mps_distribution} is trivially satisfied. 
We assume that MPS with $n-1$ tensor nodes, the contracted tensor $B^{s_{1}\cdots s_{n-1}}_{\alpha_{1}\alpha_{n}}$ belongs to the Normal distribution $\mathcal{N}(\mathbf{0}, (\alpha_{1}\alpha_{n})^{-\frac{1}{2}}(\prod_{i}^{n-1}\sigma_{i}^{2})\mathbb{I}^{s_{1}\cdots s_{n}}_{\alpha_{1}\alpha_{n}}))$. We know the following relation, 
\begin{align}
B^{s_{1}\cdots s_{n}}_{\alpha_{1}\alpha_{n+1}} = \sum_{\alpha_{n}}B^{s_{1}\cdots s_{n-1}}_{\alpha_{1}\alpha_{n}}A^{s_{n}}_{\alpha_{n}\alpha_{n+1}},
\end{align}
then by central limit theorem, as $\alpha_{n} \to \infty$,  
\begin{align}
    B^{s_{1}\cdots s_{n}}_{\alpha_{1}\alpha_{n+1}} \xrightarrow[]{\text{Dist.}}\mathcal{N}(\mathbf{0}, (\alpha_{1}\alpha_{n+1})^{-\frac{1}{2}}(\prod_{i}^{n}\sigma_{i}^{2})\mathbb{I}^{s_{1}\cdots s_{n}}_{\alpha_{1}\alpha_{n+1}})).
\end{align}
We note that all the components of $B^{s_{1}\cdots s_{n}}_{\alpha_{1}\alpha_{n+1}}$ are iid.
At last we contract the indices $\alpha_{1}$ and $\alpha_{n+1}$ in $B^{s_{1}\cdots s_{n}}_{\alpha_{1}\alpha_{n+1}}$ and set the bond dimensions $\alpha_{1}$ and $\alpha_{n+1}$ to infinity, so we get
\begin{align}
    B^{s_{1}\cdots s_{n}} \xrightarrow[]{\text{Dist.}}\mathcal{N}(\mathbf{0}, (\prod_{i}^{n}\sigma_{i}^{2}\mathbb{I}^{s_{1}\cdots s_{n}})).
\end{align}
\end{proof}

\noindent
{\bf Proposition 2.}
{\it
For the outputs of MPS $\Psi(\mathbf{x})$ defined in \ref{eq: mps_outputs}, as $\alpha_{1}, \cdots, \alpha_{n} \to \infty$ sequentially, 
\begin{align}
    \Psi(\mathbf{x}) \xrightarrow[]{\text{Dist.}}\mathcal{N}(\mathbf{0}, \prod_{i}^{n}(\sum_{j}\phi^{j}(x_{i})^{2})\sigma_{i}^{2}).
\end{align}
}
\begin{proof}
By the definition of MPS, we know $\Psi(\mathbf{x}) = \sum_{\{s_{i}\}}B^{s_{1}\cdots s_{n}}\prod_{j}\phi^{s_{j}}(x_{j})$.
According to \textbf{Proposition} $1$, we can show $\Psi(\mathbf{x})$ belongs to Normal distribution and also 
\begin{align*}
    \text{Var}[\Psi(\mathbf{x})] &= \sum_{\{s_{i}\}}\text{Var}[B^{s_{1}\cdots s_{n}}]\prod_{i}\phi^{s_{i}}(x_{i})\\
    &=\sum_{\{s_{i}\}}\prod_{i}\sigma_{i}^{2}\mathbb{I}^{s_{1}\cdots s_{n}}\prod_{i}\phi^{s_{i}}(x_{i})\\
    &=\prod_{i}\sum_{j}\phi^{j}(x_{i})^{2}\sigma_{i}^{2}
\end{align*}
\end{proof}

\noindent
{\bf Theorem 1.}
{\it
As the bond dimensions $\alpha_{1}, \cdots, \alpha_{n} \to \infty$ sequentially, the map $\Psi: \mathbf{x} \to \Psi(\mathbf{x})$ defined in \ref{eq: mps_outputs} on a data set $\Omega = \{(\mathbf{x}^{(i)}, \mathbf{y}^{(i)}), i\in\{1, \cdots, m\}\}$ converges to the GP, 
\begin{align}
    &\Psi \sim \text{GP}(\mu, \Sigma),\\
    &\mu = 0,\\
    &\Sigma(\mathbf{x}, \mathbf{x}^{\prime}) = \prod_{i}|s_{i}|\sigma_{i}^{2}\phi^{i}(x_{1})\cdot\phi^{i}(x^{\prime}_{1}),
\end{align}
where $\mu(\cdot)$ is the mean function and $\Sigma(\cdot, \cdot)$ is the covariance function. 
}
\begin{proof}
\label{pf: proof_of_the_1}
After contracting all the bond dimensions in MPS, we get 
\begin{align}
    \Psi(\mathbf{x}) = \sum_{\{s_{1}\cdots s_{n}\}}B^{s_{1}\cdots s_{n}}\Phi^{s_{1}\cdots s_{n}}(\mathbf{x}).
\end{align}
So the mean function $\mu(\cdot)$ is constant zero as 
\begin{align}
 \mathbb{E}[\Psi(\cdot)] &= \sum_{\{s\}}\mathbb{E}[B^{s_{1}\cdots s_{n}}]\Phi^{s_{1}\cdots s_{n}}(\cdot) = 0,   
\end{align}
and the covariance function $\Sigma(\mathbf{x}, \mathbf{x}^{\prime})$ is
\begin{align}
\mathbb{E}[\Psi(\mathbf{x})\Psi(\mathbf{x}^{\prime})] &= \sum_{\{s, s^{\prime}\}}\mathbb{E}[B^{s_{1}\cdots s_{n}}B^{s^{\prime}_{1}\cdots s^{\prime}_{n}}]\Phi^{s_{1}\cdots s_{n}}(\mathbf{x})\Phi^{s^{\prime}_{1}\cdots s^{\prime}_{n}}(\mathbf{x}^{\prime})\\
&=\sum_{\{s, s^{\prime}\}}(\prod_{i}\sigma^{2}_{i}\delta_{s_{i}, s^{\prime}_{i}})\Phi^{s_{1}\cdots s_{n}}(\mathbf{x})\Phi^{s^{\prime}_{1}\cdots s^{\prime}_{n}}(\mathbf{x}^{\prime})\\
&=\prod_{i}\sigma_{i}^{2}\phi^{i}(x_{i})\cdot\phi^{i}(x^{\prime}_{i}).
\end{align}
\end{proof}

\section*{B. Asymptotics of NTK of MPS}
We know MPS converges to a GP as the bond dimensions go to infinity by Theorem~\ref{the: infinite_mps_gp}. Based on this preparation, we can start to analyze the dynamics of infinite MPS. So we will calculate the NTK of MPS in the infinite limit as Theorem~\ref{eq: ntk_mps}. At first we recall the Theorem $3.1$, 

\noindent
{\bf Theorem 2.}
{\it
For MPS with the set-up as $A^{s_{i}}_{\alpha_{i}\alpha_{i+1}} \sim$ $\mathcal{N}(\mathbf{0}, \frac{\sigma_{1}^{2}}{\sqrt{|\alpha_{1}||\alpha_{2}|}}\mathbb{I}^{s_{1}}_{\alpha_{1}\alpha_{2}}))$, as the bond dimensions $\{\alpha_{i}, i\in\{1, \cdots, n\}\}$ goes to infinity sequentially, the NTK $K_{ij}(\mathbf{x}^{(i)}, \mathbf{x}^{(j)})$ of MPS converges to a static matrix in probability, 
\begin{align}
    K_{ij}(\mathbf{x}^{(i)}, \mathbf{x}^{(j)}) \xrightarrow[]{\text{Prob.}}\sum_{k}\phi(x^{(i)}_{k})\cdot\phi(x^{(j)}_{k})\prod_{l=1; l \neq k}^{n}\sigma_{l}^{2}\phi(x^{(i)}_{l})\cdot\phi(x^{(j)}_{l}).
\end{align}
The NTK here is defined as 
\begin{align}
    K_{ij}(\mathbf{x}^{(i)}, \mathbf{x}^{(j)}) =\mathbf{\eta}\odot\frac{\partial{\Psi(\mathbf{x}^{(i)})}}{\partial{\mathbf{w}}}\otimes\frac{\partial{\Psi(\mathbf{x}^{(j)})}}{\partial{\mathbf{w}}}, 
\end{align}
where $\eta_{ij} = (|\alpha_{i}||\alpha_{j}|)^{-1/2}$.
}
\begin{proof}
\label{pf: proof_of_the_2}
The derivatives of MPS $\Psi(\mathbf{x}^{(i)})$ are 
\begin{align}
\label{eq: outcome_dev}
    &\nabla\Psi(\mathbf{x}^{(i)}|\{\mathbf{A}\}) = 
    \sum_{k}\nabla_{k}\Psi(\mathbf{x}^{(i)}|\{\mathbf{A}\})\hat{\textbf{e}}_{\text{k}},\\
    &\nabla_{k}\Psi(\mathbf{x}^{(i)}|\{\mathbf{A}\}) = A^{s_{1}}_{\alpha_{1}\alpha_{2}}\cdots \bar{A}^{s_{k}}_{\alpha_{k}\alpha_{k+1}}\cdots A^{s_{n}}_{\alpha_{n}\alpha_{1}}\Phi^{s_{1}\cdots s_{n}}(\mathbf{x}^{(i)}),
\end{align}
where $\bar{A}^{s_{i}}_{\alpha_{i}\alpha_{i+1}}$ represents that the corresponding tensor is removed from the tensor chain and then three free indices appear, namely $(\nabla\Psi(\mathbf{x}^{(i)}|\{\mathbf{A}\}))^{s_{i}}_{\alpha_{i}\alpha_{i+1}}$.

We can calculate the Gram matrix of the Jocobian of $\Psi(\cdot)$ with respect to the tensors $\mathbf{A}$ in the tensor chain as 
\begin{align*}
     \frac{\partial{\Psi(\mathbf{x}^{(i)})}}{\partial{\mathbf{A}}}\otimes\frac{\partial{\Psi(\mathbf{x}^{(j)})}}{\partial{\mathbf{A}}} &= \sum_{k}\sum_{\{s_{k}, \alpha_{k}, \alpha_{k+1}\}}(\nabla_{k}\Psi(\mathbf{x}^{(i)}|\{\mathbf{A}\}))^{s_{k}}_{\alpha_{k}\alpha_{k+1}}(\nabla_{k}\Psi(\mathbf{x}^{(j)}|\{\mathbf{A}\}))^{s_{k}}_{\alpha_{k}\alpha_{k+1}}\\
    &=\sum_{k}\sum_{\{s_{k}, \alpha_{k}, \alpha_{k+1} \}}B^{s_{1}\cdots \bar{s}_{k}\cdots s_{n}}_{\alpha_{k}\alpha_{k+1}}B^{s^{\prime}_{1}\cdots \bar{s}_{k}\cdots s^{\prime}_{n}}_{\alpha_{k}\alpha_{k+1}}\Phi^{s_{1}\cdots s_{k}\cdots s_{n}}(\mathbf{x}^{(i)})\Phi^{s^{\prime}_{1}\cdots s_{k}\cdots s^{\prime}_{n}}(\mathbf{x}^{(j)})\\ 
    &\xrightarrow[]{\text{Prob.}}\sum_{k}\sum_{\{s_{k}\}}|\alpha_{k}||\alpha_{k+1}|\mathbb{E}[B^{s_{1}\cdots \bar{s}_{k}\cdots s_{n}}_{\alpha_{k}\alpha_{k+1}}B^{s^{\prime}_{1}\cdots \bar{s}_{k}\cdots s^{\prime}_{n}}_{\alpha_{k}\alpha_{k+1}}] \Phi^{s_{1}\cdots s_{k}\cdots s_{n}}(\mathbf{x}^{(i)})\Phi^{s^{\prime}_{1}\cdots s_{k}\cdots s^{\prime}_{n}}(\mathbf{x}^{(j)})\\
    &=\sum_{k}\sum_{\{s_{k}\}}|\alpha_{k}||\alpha_{k+1}|\mathbb{V}[B^{s_{1}\cdots \bar{s}_{k}\cdots s_{n}}_{\alpha_{k}\alpha_{k+1}}]\delta_{s_{1}s^{\prime}_{1}}\cdots\delta_{s_{k-1}s^{\prime}_{k-1}}\delta_{s_{k+1}s^{\prime}_{k+1}}\cdots\delta_{s_{n}s^{\prime}_{n}}\\
    &\quad \Phi^{s_{1}\cdots s_{k}\cdots s_{n}}(\mathbf{x}^{(i)})\Phi^{s^{\prime}_{1}\cdots s_{k}\cdots s^{\prime}_{n}}(\mathbf{x}^{(j)})\\
    &=\sum_{k}|\alpha_{k}||\alpha_{k+1}|\frac{1}{\sqrt{|\alpha_{k}||\alpha_{k+1}|}}\phi(x^{(i)}_{k})\cdot\phi(x^{(j)}_{k})\prod_{l=1; l \neq k}^{n}\sigma_{l}^{2}\phi(x^{(i)}_{l})\cdot\phi(x^{(j)}_{l})\\
    &=\sum_{k}(|\alpha_{k}||\alpha_{k+1}|)^{1/2}\phi(x^{(i)}_{k})\cdot\phi(x^{(j)}_{k})\prod_{l=1; l \neq k}^{n}\sigma_{l}^{2}\phi(x^{(i)}_{l})\cdot\phi(x^{(j)}_{l}),
\end{align*} 
then we have
\begin{align}
    &K_{ij}(\mathbf{x}^{(i)}, \mathbf{x}^{(j)}) =\mathbf{\eta}\odot\frac{\partial{\Psi(\mathbf{x}^{(i)})}}{\partial{\mathbf{A}}}\otimes\frac{\partial{\Psi(\mathbf{x}^{(j)})}}{\partial{\mathbf{A}}}\\
    &\xrightarrow[]{\text{Prob.}}\sum_{k}\phi(x^{(i)}_{k})\cdot\phi(x^{(j)}_{k})\prod_{l=1; l \neq k}^{n}\sigma_{l}^{2}\phi(x^{(i)}_{l})\cdot\phi(x^{(j)}_{l})
\end{align}
\end{proof}
\begin{remark}
To achieve consistent asymptotic behavior, we need to assume the learning rate $\eta$ to be $O((|\alpha_{i}||\alpha_{j}|)^{-1/2})$ which is different from the NTK of ANN, where the learning rate is assumed to be constant but each layer is rescaled with factor $\frac{1}{\sqrt{n}}$.
\end{remark}
Here we prove the "lazy training" phenomena proposed in Lemma~\ref{le: lazy_training}.

\noindent
{\bf Lemma 3.5}
{\it
For a MPS with the set-up as above, as the bond dimensions go to infinity, we have the following relation 
\begin{align}
    \lim_{\alpha_{1}\cdots \alpha_{n}\to \infty}\sup_{t\in[0, T]}{|A^{s_{i}}_{\alpha_{i}\alpha_{i+1}} - A^{s_{i}}_{\alpha_{i}\alpha_{i+1}}(0)|} \xrightarrow[]{\text{Prob.}} 0.
\end{align}
}
\begin{proof}
\label{pf: pf_of_zero_variations}
We can write down the dynamics of the tensors as 
\begin{align}
\frac{d}{dt}A^{s_{i}}_{\alpha_{i}\alpha_{i+1}} = \langle \partial_{A}\Psi, d\rangle_{p^{\text{in}}}.
\end{align}
By Proposition~\ref{prop: tensor_chain_convergence}, as $\alpha_{1}\cdots \alpha_{n} \to \infty$, we get 
\begin{align}
    \partial_{A}\Psi(\mathbf{x})\xrightarrow[]{\text{Dist.}}\mathcal{N}(\mathbf{0}, (\alpha_{1}\alpha_{n+1})^{-\frac{1}{2}}(\prod_{i}^{n}\sigma_{i}^{2})\mathbb{I}^{s_{1}\cdots s_{n}}_{\alpha_{1}\alpha_{n+1}})(\Phi^{s_{1}\cdots s_{n}}(\mathbf{x}))^{2})\xrightarrow[]{\text{Prob.}}\mathbf{0}.
\end{align}
Then we know the variations $|\Delta A^{s_{i}}_{\alpha_{i}\alpha_{i+1}}|$ of the tensors converge to zero in probability as all the bond dimensions go to infinity sequentially.
\end{proof}
By the dynamics of the loss $\mathcal{L}$ is $\partial_{t}\mathcal{L} = -||d_{\Psi}||_{\text{K}}$ as shown in equation~\ref{eq: loss_dynamics}. So it is proved that the derivative of the objective is negative if we can show that $K(\cdot, \cdot)$ is positively definite. So we will prove following Proposition by showing that the Mercer's condition is satisfied by $K_{ij}$,

\begin{proposition}
The NTK $K_{ij}(\mathbf{x}^{(i)}, \mathbf{x}^{(j)})$ of infinite bond dimensional MPS is positive definite.
\end{proposition}
\begin{proof}
\label{pf: ntk_mps_positive}
For an arbitrary collection of coefficients $\{c_{i}, i=1, \cdots, d\}$,
\begin{align}
\sum_{i, j = 1}^{d}c_{i}c_{j}K_{t}(\mathbf{x}^{(i)}, \mathbf{x}^{(j)}) &= \sum_{k}\sum_{i, j = 1}^{d}c_{i}c_{j}\phi(x^{(i)}_{k})\cdot\phi(x^{(j)}_{k})\prod_{l=1; l \neq k}^{n}\sigma_{l}^{2}\phi(x^{(i)}_{l})\cdot\phi(x^{(j)}_{l})\\
&=\sum_{k}(\prod_{l=1; l\neq k}^{n}\sigma_{l}^{2})\sum_{i, j=1}^{d}c_{i}c_{j}\prod_{l=1}^{n}\phi(x^{(i)}_{l})\phi(x^{(j)}_{l})\\
&=\sum_{k}(\prod_{l=1; l\neq k}^{n}\sigma_{l}^{2})\sum_{i, j=1}^{d}c_{i}c_{j}\prod_{l=1}^{n}\phi(x_{l}^{(i)})\prod_{m=1}^{n}\phi(x_{m}^{(j)})\\
&=\sum_{k}(\prod_{l=1; l\neq k}^{n}\sigma^{2}_{l})(\sum_{i=1}^{d}c_{i}\prod_{l=1}^{n}\phi(x_{l}^{(i)}))^{2}
\end{align}
We show that $\sum_{i, j=1}^{d}c_{i}c_{j}K_{t}(\mathbf{x}_{i}, \mathbf{x}_{j})$ is positive definite since it is always non-negative and become zero only when all $\{c_{i}, i=1, \cdots, d\}$ are zero. 
\end{proof}


\section*{C. Convergences on Born Machines}
In BM, the partition function $Z[\Psi]$ is the sum of multiple squared Normal random variables as
\begin{align}
    Z = \sum_{s^{i}} (B^{s_{1}\cdots s_{n}}\Phi)^{2}.  
\end{align}
By carefully designing the kernel function $\Phi^{s_{1}\cdots s_{n}}(\cdot)$ such that all the outputs of BM $\Psi(\mathbf{x}^{(i)})$ with different inputs $\mathbf{x}^{(i)}$ are independent. 
the partition function $Z[\Psi]$ follows the Gamma distribution. Now we prove the Proposition~\ref{prop: partition_distribution} as follows,
\begin{proof}
\label{pf: partition_function_bm}
By Theorem~\ref{the: infinite_mps_gp}, we know the outputs of BM $\Psi(\mathbf{x}^{(i)})$ on different sample points are iid random variables following Normal distribution, namely $\Psi(\mathbf{x}^{(i)}) \sim \mathcal{N}(0, \prod_{i}^{n}\sigma_{i}^{2})$, then we have
\begin{align}
    |\Psi(\mathbf{x}^{(i)})|^{2} \sim \Gamma(\frac{1}{2}, 2\prod_{i}^{n}\sigma_{i}^{2}).
\end{align}
The partition function $Z[\Psi]$ is just the sum of the squared outcomes $|\Psi^{2}(\cdot)|^{2}$ of BM and we get 
\begin{align}
    Z[\Psi] = \sum_{\{\mathbf{x}_{i}\}}|\Psi(\mathbf{x}^{(i)})|^{2} \sim \Gamma(2^{n-1}, 2\prod_{i}^{n}\sigma_{i}^{2}).
\end{align}
Moreover, by taking the large $n$ limit and using the WLLN, we have
\begin{align}
    \frac{Z}{2^{n}} \xrightarrow[]{\text{Prob.}}\mathbb{E}[|\Psi(\cdot)|^{2}] = \text{Var}[\Psi(\cdot)] = \prod_{i}^{n}\sigma_{i}^{2}.
\end{align}
\end{proof}


\end{document}